\documentclass[sigconf]{acmart}
\pdfoutput=1

\usepackage[utf8]{inputenc} 
\usepackage{hyperref}       
\usepackage{url}            
\usepackage{booktabs}       
\usepackage{nicefrac}       

\usepackage{mathtools}
\usepackage{subcaption}
\usepackage{algorithm}
\usepackage[noend]{algpseudocode}
\usepackage{multirow}
\usepackage{paralist}
\usepackage{mwe}

\newcommand{\opt}{\textsc{Opt}}
\newcommand{\cost}{\textrm{cost}}
\newcommand{\error}{\textrm{AbsError}\xspace}
\newcommand{\Ratio}{\textrm{RelError}\xspace}
\newcommand{\Loss}{\textrm{RelErrorDiff}\xspace}

\DeclareMathOperator*{\argmin}{arg\,min}
\newcommand{\calL}{\mathcal{L}}
\newcommand{\zbar}{\overline{z}}
\newcommand{\ybar}{\overline{y}}

\newtheorem{thm}{Theorem}
\theoremstyle{definition}
\newtheorem{defn}[thm]{Definition}
\newtheorem{remark}[thm]{Remark}
\newtheorem{corr}[thm]{Corollary}
\newtheorem{obs}[thm]{Observation}
\newtheorem{lemma}[thm]{Lemma}

\usepackage[draft,inline,nomargin,index]{fixme}

\fxsetup{theme=color,mode=multiuser,inlineface=\itshape,envface=\itshape}
\FXRegisterAuthor{sv}{asv}{\colorbox{gray!10!white}{\color{black}Suresh}}
\FXRegisterAuthor{ma}{ama}{\colorbox{blue!10!white}{\color{black}Mohsen}}
\FXRegisterAuthor{ab}{aab}{\colorbox{red!10!white}{\color{black}Aditya}}
\fxusetargetlayout{color}

\AtBeginDocument{%
  \providecommand\BibTeX{{%
    \normalfont B\kern-0.5em{\scshape i\kern-0.25em b}\kern-0.8em\TeX}}}

\setcopyright{none}
\renewcommand\footnotetextcopyrightpermission[1]{} 
\begin{document}

\title[Fair Clustering]{Fair Clustering via Equitable Group Representations}
\titlenote{This research was funded in part by the NSF under grants IIS-1633724 and CCF-2008688.}

\author{Mohsen Abbasi}
\email{mohsen@cs.utah.edu}
\affiliation{
  \institution{University of Utah}
}

\author{Aditya Bhaskara}
\email{bhaskara@cs.utah.edu}
\affiliation{
  \institution{University of Utah}
}

\author{Suresh Venkatasubramanian}
\email{suresh@cs.utah.edu}
\affiliation{%
  \institution{University of Utah}
}

\renewcommand{\shortauthors}{Abbasi, et al.}

\begin{abstract}
What does it mean for a clustering to be fair? One popular approach seeks to ensure that each cluster contains groups in (roughly) the same proportion in which they exist in the population. The normative principle at play is balance: any cluster might act as a representative of the data, and thus should reflect its diversity.  

But clustering also captures a different form of representativeness. A core principle in most clustering problems is that a cluster center should be representative of the cluster it represents, by being ``close" to the points associated with it. This is so that we can effectively replace the points by their cluster centers without significant loss in fidelity, and indeed is a common ``use case'' for clustering.  For such a clustering to be fair, the centers should ``represent'' different groups equally well. We call such a clustering a group-representative clustering.  

In this paper, we study the structure and computation of group-representative clusterings. We show that this notion naturally parallels the development of fairness notions in classification, with direct analogs of ideas like demographic parity and equal opportunity. We demonstrate how these notions are distinct from and cannot be captured by balance-based notions of fairness.  We present approximation algorithms for group representative $k$-median clustering and couple this with an empirical evaluation on various real-world data sets. We also extend this idea to facility location, motivated by the current problem of assigning polling locations for voting. 
\end{abstract}

	
\maketitle

\section{Introduction}
\label{sec:intro}
Growing use of automated decision making has sparked a debate concerning bias, and what it means to be fair in this setting. As a result, an extensive literature exists on algorithmic fairness, and in particular on how to define fairness for problems in supervised learning \cite{Dwork12Fairness,Romei13Multidisciplinary,Feldman2015DisparateImpact,hardt2016equality,arvindtutorial,mitchell2018predictionbased}. 
However, these notions are not readily applicable to \emph{unsupervised} learning problems such as clustering. One reason is that unlike in the supervised setting, a well-defined notion of ground truth does not exist in such problems. In 2017, \cite{chierichetti2017fair} proposed the idea of \emph{balance} as a notion of fairness in clustering. Given a set of data points with a type assigned to each one, balance asks for a clustering where each cluster has roughly the same proportion of types as the overall population. This definition spawned a flurry of research on efficient algorithms for fair clustering \cite{chierichetti2017fair, kleindessner2019fair, kleindessner2019guarantees, chen2019proportionally, schmidt2018fair, ahmadian2019clustering, rosner2018privacy}. Further work by other researchers has extended this definition, but with the same basic principle of proportionality \citep{abraham2019fairness, backurs2019scalable, bercea2018cost, huang2019coresets, bera2019fair, wang2019towards, jing2019clustering}.

Balance draws its meaning from the perspective of clustering as a generalization of classification from two to many categories. If we select individuals into multiple categories where each category has some associated benefits (or harms), balance asks that different groups receive these benefits or harms in similar proportions. For example, a tool that assigns individuals to different categories based on the loan packages being offered might attempt to ensure demographic balance across all categories. 

But an important purpose of clustering is to find \emph{representatives} for points by grouping them and choosing a representative (like a cluster center). Consider for example the problem of redistricting -- where the goal is to partition a region into districts, each served by one representative who can speak to the concern of the district. The criterion of balance, applied (say) to partisan affiliation, will result in districts that have a proportional number of residents associated with each party. This is unfortunately the worst kind of redistricting! It is a form of gerrymandering known as cracking, and results in the majority party being able to control representation in \emph{all} the districts. 

In fact the vast majority of clustering formulations focus on the problem of quality representation. And the quality of representation is usually measured by some form of distance to the chosen representative -- the greater the distance, the poorer the representation. There are serious implications for fairness as measured in terms of "access", and we illustrate this with an example of current concern. 

Consider the placement of polling locations. A study showed that in the 2016 US presidential election, voters in predominantly black neighborhoods waited 29 percent longer at polling locations, than those in white neighborhoods \citep{chen2019racial}, and in 2020 we are currently seeing polling locations removed or merged because of the pandemic. The clustering here is the induced clustering where each resident is associated with a specific polling location (the representative) and the quality of representation can be measured as a function both of the distance to the polling location and the waiting time at the location itself. 

It is easy to see that balance criteria cannot accurately capture the goal of representation equity. 
To illustrate why, see the example presented in Figure \ref{fig:balance}, which shows a balance-preserving $k$-means clustering on the left for two groups denoted by the colors red and blue, and regular $k$-means clustering on the right. Here, the number of red points is larger than the other group. Therefore, each cluster center is chosen close to its respective red group's centroid. As a result, red points are better represented by chosen centers compared to blue points. 
\begin{figure}
\centering
\includegraphics[width=\columnwidth]{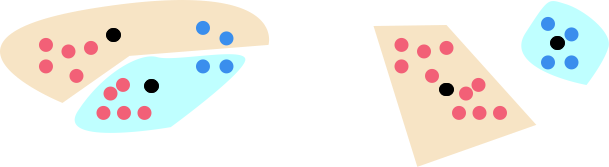}
\caption
{Balance is preserved in the left hand figure but the centers are much more representative of the red points. In contrast the clustering on the right represents both groups comparably}    
\label{fig:balance}
\end{figure}

\subsection{Our contributions}

In this paper we introduce a new way to think about fairness in clustering based on the idea of equity in representation for groups. We present a number of different ways of measuring representativeness and interestingly, show that they 
\begin{inparaenum}[a)]
\item naturally parallel standard notions of fairness in the supervised learning literature and 
\item are incompatible in the same way that different notions of fairness in supervised learning are. 
\end{inparaenum}
We present algorithms for computing fair clusterings under these notions of fairness. These algorithms come with formal approximation guarantees -- we also present an empirical study of their behavior in the context of the problem of polling location placement. 

A second key contribution of this work takes advantage of the relation between clustering and the associated problem of \emph{facility location}. In most clustering problems, the number of clusters that the algorithm must produce is fixed in advance. In the context of polling, this is equivalent to specifying the number of polling locations to be established ahead of time. But we can associate a cost with \emph{opening} a new facility (or cluster center), giving us a choice of either assigning a point to a center that might be far away or opening a new center that might be closer. This is the well known problem of facility location\cite{shmoys1997approximation} -- it is closely related to clustering\footnote{We can think of facility location as the Lagrangian relaxation of clustering where the hard constraint on the number of clusters is replaced by a term in objective function.}. Again returning to the context of polling, the framing in terms of facility location allows for the possibility of creating new polling locations if the cost of opening can be ``paid for'' in terms of improved access. We define what representation equity means in the context of facility location and present algorithms (and experimental support) as well. Our algorithms also allow us to incorporate {\em load balancing} between different facilities. In the context of polling, we can ensure that no more than $U$ voters are assinged to a polling location, for a user-specified threshold $U$. 


\subsection{Related Work}
\label{sec:related-work}
Chierichetti et al. \cite{chierichetti2017fair} introduced balance as a fairness constraint in clustering for two groups. Considering the same setting with binary attribute for groups, Backurs et al. improved the running time of their algorithm for fair $k$-median \citep{backurs2019scalable}. Rösner and Schmidt \citep{rosner2018privacy} proposed a constant-factor approximation algorithm for fair $k$-center problem with multiple protected classes. Bercea et al.  \citep{bercea2018cost} proposed bicriteria constant-factor approximations for several classical clustering objectives and improved the results of R\"osner and Schmidt. Bera et al. \citep{Bera19} generalized previous works by allowing maximum over- and minimum under-representation of groups in clusters, and multiple, non-disjoint sensitive types in their framework. Other works have studied multiple types setting \citep{wang2019towards}, multiple, non-disjoint types \citep{huang2019coresets} and cluster dependent proportions \citep{jing2019clustering}.

In a different line of work, Ahmadian et al. \citep{ahmadian2019clustering} studied fair $k$-center problem where there is an upper bound on maximum fraction of a single type within each cluster. Chen et al. \citep{chen2019proportionally} studied a variant of fair clustering problem where any large enough group of points with respect to the number of clusters are entitled to their own cluster center, if it is closer in distance to all of them.

A large body of works in the area of algorithmic fairness have focused on ensuring fair representation of all social groups in the machine learning pipeline \citep{bolukbasi2016man, samadi_price_2018, abbasi2019fairness}. Recent work by Mahabadi et al. \citep{mahabadi2020individual}, studies the problem of individually fair clustering, under the fairness constraint proposed by Jung et al. \citep{jung2019center}. In their framework, if $r(x)$ denotes the minimum radius such that the ball of radius $r(x)$ centered at $x$ has at least $n/k$ points, then at least one center should be opened within $r(x)$ distance from $x$. In recent independent work, Ghadiri et al.~\citep{ghadiri2020fair} propose a fair $k$-means objective similar to one of our objectives, and study a variant of Lloyd’s algorithm for determining cluster centers.
\section{Fair Clustering}

In this section we introduce notions of fair clustering that are rooted in the idea of equitable representation. To that end, we introduce different ways to measure the cost of group representation. We can then define a \emph{fair} clustering. 

We start with some basic definitions. Given a set of points $X$, a clustering is a partitioning\footnote{It is possible to define so-called \emph{soft} clusterings in which a point might be assigned fractionally to multiple clusters. We do not consider such clusterings in this paper.} of $X$ into clusters $C_1, C_2, \dots, C_k$. For most of the paper we will consider clustering objectives that satisfy the \emph{Voronoi property}: the optimal assignment for a point is the cluster center nearest to it. This includes the usual clustering formulations like $k$-center, $k$-means and $k$-median. For such clusterings, the cluster center defines the cluster and thus we can represent a clustering more compactly as the set of cluster \emph{centers} $C = \{c_1, c_2, \dots, c_k\}$. The cost of a clustering $C$ of a set of points $P$ is given by the function $\cost_C(P)$. For any subset of points $S$, we denote the cost of assigning $S$ to cluster centers in a given clustering $C$ as $\cost_C(S)$. Finally, given a cost function $\cost$ and a set of points $X$ we denote the set of centers in an optimal clustering of $X$ by $\opt_\cost(X)$. When the context is clear, we will drop the subscript and merely write this as $\opt(X)$. As usual, an $\alpha$-approximation algorithm is that one that returns a solution that is within an $\alpha$-multiplicative factor of $\opt(X)$.

Each point in $X$ is associated with one of $m$ \emph{groups} (e.g., demographic) that we wish to ensure fairness with respect to. We define the subset of points in group $i$ as $X_i$. 
\begin{defn}(Fair Clustering). 
\label{def:fair}
Given $m$ groups $X_1, X_2, \dots, X_m$,  fair clustering minimizes the maximum average (representation) cost across all groups:
\[\argmin_{C \in \mathcal{C}}~ \text{max} \left( \frac{1}{|X_1|}\cost_C(X_1), \dots, \frac{1}{|X_m|}\cost_C(X_m) \right) \]
where $\mathcal{C}$ is the set of all possible choices of cluster centers.
\end{defn}

Note that we are not trying to force all groups to have the same representation cost; that constraint can be trivially satisfied by ensuring all groups have a large cost (i.e via a poor representation). Rather, we want to ensure all groups have good representation while still ensuring that the gap between groups (measured by the group with the maximum cost) is small. This also distinguishes our definition from proposals for clustering cost that try to minimize the maximum cost across \emph{clusters} \citep{tzortzis2014minmax} instead of across \emph{groups.}

\subsection{Quality of group representation}
\label{sec:notions}

We now introduce different ways to measure the group representation cost $\cost_C(X_i)$.
\paragraph{Absolute Representation Error}
In supervised learning, \emph{statistical parity} captures the idea that groups should have similar outcomes. Rephrasing, it says that groups should be \emph{represented equally well in the output}. 
In the case of binary classification, statistical parity requires 
\[\Pr(h(x) = 1|S = a) = \Pr(h(x) = 1| S = b)\]
for two groups $a$ and $b$, where $S$ denotes the sensitive attribute. When clustering, the equivalent notion of statistical parity asks that cluster centers represent all groups equally well, regardless of their potentially different distributions. More specifically, the average distance between members of a group and their respective cluster centers should look the same across groups. Motivated by this, we introduce the following definition of representation cost. 

\begin{defn}[\error]
\label{def:abs-error}
  The absolute (representation) error of a clustering is defined as
\[\error_{C}(X) = \sum_{x \in X}d(x,C), \]
where $X$ is a set of points, $C$ is a set of centers and $d(x,C)$ is a an arbitrary distance function between $x$ and nearest center to it in $C$.
\end{defn}

An \error{}-fair clustering is a fair clustering that uses \error to measure group representation cost in Definition~\ref{def:fair}. 

\paragraph{Relative representation error}
\error does not take underlying distributions of demographic groups into account. However, in cases where different  groups have drastically different underlying distributions, it may be necessary to acknowledge such a difference. Consider the example provided in Figure \ref{fig:compare}, where one demographic group (orange) has a much smaller variance compared to the other (blue). Assume that the within-group distances for the orange group can be ignored compared to distance $d$. An \error-fair clustering picks $C_{\error{} \text{Fair}}$ as the center which minimizes the maximum average \error. However, such a clustering may seem unfair, as it induces a large cost on the orange group compared to its ``optimal'' representation cost, which is close to zero if the orange group's center is picked.

\begin{figure}
\centering
\includegraphics[width=\columnwidth]{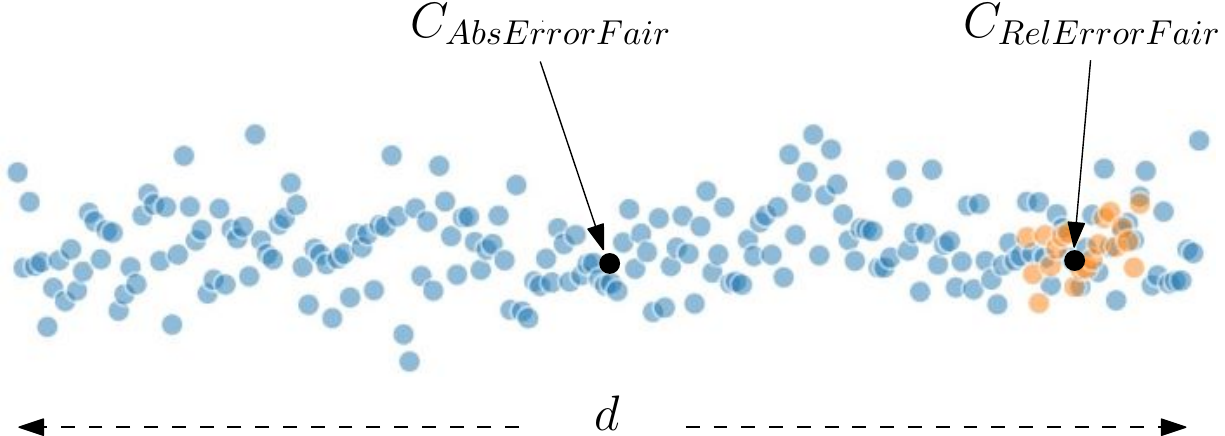}
\caption
{For groups with different underlying distributions, \error and \Ratio lead to different locations for the center.}    
\label{fig:compare}
\end{figure}
 
 The issue of differences in base distributions motivates fairness measures like \emph{equality of opportunity} based on balancing error rates rather than outcomes \cite{hardt2016equality}. As with statistical parity, we can define a natural analog in the context of representation. 
 Rather than look at the error in representation in absolute terms, we compare the average distance between members of a group and their respective cluster centers, to the corresponding ``optimal'' value for that group (if we only clustered the members of that group).

This \emph{relative} measure of representation error yields the following definition.

\begin{defn}[\Ratio]
  The relative (representation) error of a clustering is given by 
\[\Ratio_C(X) =  \frac{\sum_{x \in X}d(x,C)}{\sum_{x \in X}d(x,\opt(X))}\]
where $X$ is a set of points, $C$ is a set of centers and $d(x,C)$ is a an arbitrary distance function between $x$ and nearest center to it in $C$.
\end{defn}

One can also try to capture the relative error via a difference instead of a division.
\[\Loss_C(X) = \frac{1}{|X|} \left( \sum_{x \in X}d(x,C)-\sum_{x \in X}d(x,\opt(X)) \right) \]
This is similar to the formulation used by \cite{samadi_price_2018} in their work on fair PCA. However, in the case of clustering, this quantity can be NP-hard to compute even for a given set of centers $C$. This is because $\sum_x d(x, \opt(X))$ cannot be computed exactly. For this reason, we will not discuss this formulation further. 

\subsection{How the different measures of fairness compare} 

We have drawn on an analogy to fairness in supervised learning to formulate two measures of group representation fairness. We now make the informal case that the analogy also carries over to the incompatibility between the two notions, similar to the well-known result in the supervised learning setting \citep{kleinberg2016inherent, chouldechova2017fair}. 
To see this, consider a point set $X$ with two defined groups $A$ and $B$. Without loss of generality, assume that $B$ can be clustered better than $A$ i.e that $\cost_{opt_A}(A) > \cost_{opt_B}(B)$. Now consider Figure \ref{fig:impossible} where any clustering of $X$ is represented as a point whose $x$-coordinate is the cost $\cost_C(A)$ and whose $y$-coordinate is the cost $\cost_C(B)$. The $\error$ fair line represents all clusterings where the two groups have equal $\error$ costs and the 
$\Ratio$ fair line (which passes through the point $(\cost_{opt_A}(A),\cost_{opt_B}(B))$) represents all clusterings where the two groups have equal $\Ratio$ costs. 

Consider an optimal \emph{unconstrained} clustering represented by point $c$. Consider the position of a fair clustering under either of the above two measures relative to $c$. Clearly such a clustering would not be in either of the areas marked in red dots (because either the costs for both groups increase, making it inferior to $c$ or both costs decrease, contradicting the optimality of $c$). Further note that a clustering that is better with respect to \error must be closer to \error fair line compared to $c$, and one that is better with respect to \Ratio must be closer to \Ratio line compared to $c$. The only way in which these two notions can coincide in a single clustering is if either both groups have the same optimal cost ($\cost_{opt_A}(A) = \cost_{opt_B}(B)$) in which case the two lines coincide (analogous to the two groups having the same base distribution), or if the optimal clustering happens to achieve zero cost for both groups (analogous to the point set admitting perfect clustering). This (informal) reasoning hints at a more general incompatibility theorem (for $>2$ clusters) analogous to the works of~\citep{kleinberg2016inherent, chouldechova2017fair}, which we leave as an open direction.

\begin{figure}
\centering
\includegraphics[width=\columnwidth]{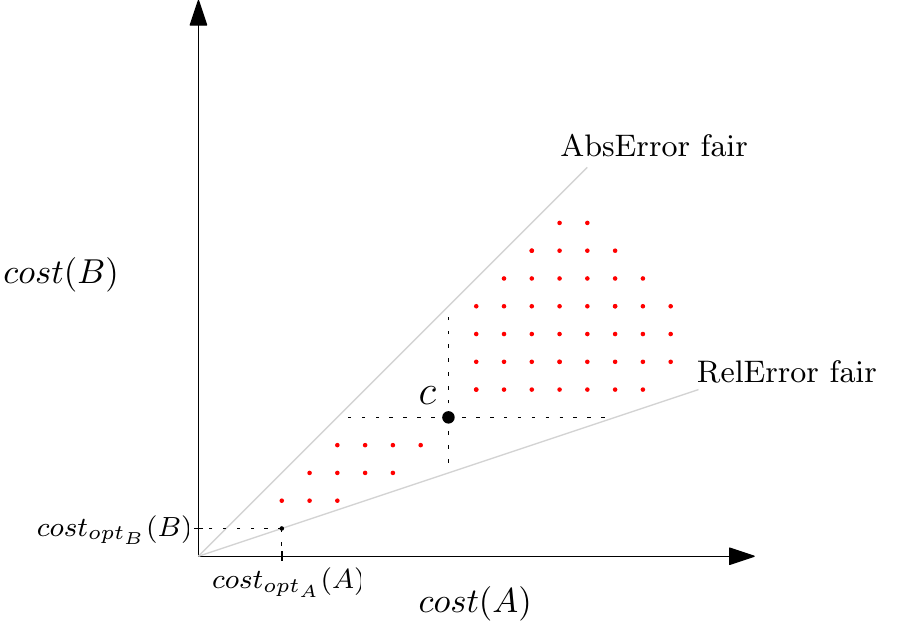}
\caption
{Unless the two groups have the same base rates, we cannot achieve fairness with respect to $\error$ and $\Ratio$.}
\label{fig:impossible}
\end{figure}

\subsection{Fairness using standard clustering methods}
\label{sec:algor-fair-clust}

Sections~\ref{sec:lp} and~\ref{sec:facloc} will present algorithms for fair clustering  under the measures of fairness introduced above. Before that, we make a simple observation about approximation guarantees we can achieve by using standard clustering algorithms with minor modification. 

If our data $X$ is composed of groups $X_1, X_2, \dots, X_m$ as we discussed, then optimizing the cost function on $X$ would correspond to the minimization problem:
\[\min_C \sum_{i \in X_i} \cost_{C}(X_i). \]

For most known cost functions and algorithms, it turns out that adding a {\em weight} for each point still allows for efficient algorithms. Let us consider assigning a weight of $1/|X_i|$ to all the points in $X_i$. Then the weighted minimization problem is
\begin{equation}
\min_C \sum_{i \in X_i} \frac{1}{|X_i|} \cost_{C}(X_i).\label{eq:weighted-problem}
\end{equation}

The observation is now the following.
\begin{obs}
Let $C$ be a clustering produced by an $\alpha$-approximation algorithm for the problem~\eqref{eq:weighted-problem}.  Then $C$ achieves an $\alpha m$ approximation to the fair clustering problem (definition~\ref{def:fair}) with the same cost function.
\label{obs:simple}
\end{obs}
\begin{proof}
Let $\Theta$ be the objective value for the fair clustering problem, and let $C^*$ be the corresponding clustering. Then by definition, we have
\[ \max_i \frac{1}{|X_i|} \cost_{C^*} (X_i) = \Theta \implies \sum_i \frac{1}{|X_i|} \cost_{C^*} (X_i) \le m \Theta.\]
Because the solution found by the algorithm ($C$) is an $\alpha$ approximation, we have that
\[ \sum_i \frac{1}{|X_i|} \cost_{C} (X_i) \le \alpha \cdot \sum_i \frac{1}{|X_i|} \cost_{C^*} (X_i) \le m\Theta. \]
This implies that each term is $\le m\Theta$, implying the desired approximation bound.
\end{proof}

This observation implies, for example, that for the fair versions of $k$-means and $k$-median, we can use known constant factor approximation algorithms to get $O(m)$ factor approximation algorithms directly.

\section{Approximation algorithms for fair $k$-clustering}\label{sec:lp}
For the fair $k$-median and $k$-means problem, we now develop approximation algorithms by writing down a linear programming relaxation and developing a rounding algorithm.

\newcommand{\lpk}{\textsc{FairLP-\error}}
\newcommand{\lpratio}{\textsc{FairLP-\Ratio}}
\newcommand{\kmed}{\textsc{K-Med-Approx}}
\newcommand{\eps}{\epsilon}
\newcommand{\cI}{\mathcal{I}}

\subsection{Relaxation for \error{}-Fair clustering}\label{sec:lp-error}
We start by considering the $\error{}$ objective for fair clustering (definitions~\ref{def:fair} and~\ref{def:abs-error}). To recap, we have a set of points $X$ that is composed of $m$ (disjoint) groups, $X_1, X_2, \dots, X_m$. We first describe the algorithm for $k$-median, and then discuss how to adapt it to the $k$-means objective (see Remark~\ref{remark:kmeans}). 

At a high level, the algorithm aims to ``open'' a subset of the $X$ as centers, and assign the rest of the points to the closest open point. The choice of the points as well as the assignment is done by using a linear programming (LP) relaxation for the problem, and rounding the obtained solution.  The variables of the LP are as follows: for $u,v \in X$,  $z_{uv}$ is intended to denote if point $u$ is assigned to center $v$. These are called assignment variables. We also have variables $y_v$ that are intended to denote if $v$ is chosen as one of the centers (or medians).  The LP (called \lpk{}) is now the following:

\begin{align}
 \min  ~\lambda \quad & \text{ subject to} \notag\\  
  \begin{aligned}
\sum_v z_{uv} &= 1 \quad \text{for all $u$}, \notag \\
z_{uv} &\le y_v \quad \text{for all $u,v$}, \notag \\
\end{aligned}
  &\qquad
  \begin{aligned}
  \sum_{v \in X} y_v &\le k, \notag\\
0 \le z_{uv},  ~y_v &\le 1 \quad \text{for all $u, v$}.\notag \\
\end{aligned}\\
  \frac{1}{|X_i|} \cdot \sum_{u \in X_i} \sum_{v \in X} d(u, v) z_{uv} &\le \lambda \quad \text{for all groups $i$}, 
\label{eq:class-bound}
\end{align}

The only new constraint compared to the standard LP formulation for $k$-median (e.g.,~\cite{charikar2002constant}) is the constraint for all groups $i$, Eq.~\eqref{eq:class-bound}. This is to ensure that we minimize the maximum $k$-median objective over the groups. 

\begin{thm}\label{thm:igap}
The integrality gap of \lpk{} is $\ge m$.
\end{thm}
\begin{proof}
Consider an instance in which we have $m$ groups, each consisting of a single point. Formally, let $X_i = \{x_i\}$ for all $i \in [m]$. Suppose that $d(x_i, x_j) = D$ for all $i \ne j$, and let $k=m-1$.

Now, consider the fractional solution in which $y_i = 1-\frac{1}{m}$ for all $i$. Also, let $z_{ii} = 1-\frac{1}{m}$, and let $z_{ij} = 1/m$ for some $j \ne i$ (it does not matter which one). It is easy to see that this solution satisfies all the constraints. Moreover, the LP objective value is $\lambda = D/m$.

However, in any integral solution, one of the points is not chosen as a center, and thus the objective value is at least $D$. Thus the integrality gap is $\ge m$.
\end{proof}

Theorem~\ref{thm:igap} makes it difficult for an LP based approach to give an approximation factor better than $m$ (which is easy to obtain as we saw in Observation~\ref{obs:simple}. 
However, the LP can still be used to obtain a ``bi-criteria'' approximation, where we obtain a constant approximation to the objective, while opening slightly more than $k$ clusters. 

\begin{thm}\label{thm:bicriteria}
Consider a feasible solution $(z, y)$ for \lpk{} with objective value $\lambda$.  For any $\eps > 0$, there is a rounding algorithm that produces an integral solution $(\zbar, \ybar)$ such that $\sum_v \ybar_v \le k/(1-\eps)$, all other constraints of the LP are satisfied, and the objective value is $\le 2\lambda/\eps$.
\end{thm}

There has been extensive literature on the problem of rounding LPs for problems like $k$-median (see, e.g.,~\cite{li2016approximating}). However, due to our additional fairness condition, we are interested in rounding schemes with an additional property, that we define below.

\begin{defn}[Faithful rounding]\label{def:faithful}
A rounding procedure for \lpk{} is said to be $\alpha$-faithful if it takes a fractional solution $(z, y)$ and produces an integral solution $(\hat{z}, \hat{y})$ with the guarantee that for \emph{every} $u \in X$, 
\[ \sum_{v \in X} d(u,v)~ \hat{z}_{uv} \le \alpha \cdot \sum_{v \in X} d(u,v) ~z_{uv}.\]
A weaker notion that holds for some known rounding schemes is the following: a rounding algorithm is said to be \emph{$\alpha$-faithful in expectation} if
\[ \mathbb{E} \left[ \sum_{v \in X} d(u,v)~ \hat{z}_{uv} \right] \le \alpha \cdot \sum_{v \in X} d(u,v) ~z_{uv}.\]
\end{defn}

The advantage of having a faithful rounding procedure is that it automatically gives a per-group (indeed, a per-point) guarantee on the objective value, which enables us to prove the desired approximation bound.

\begin{proof}[Proof of Theorem~\ref{thm:bicriteria}]
The proof is based on the well-known ``filtering'' technique (\cite{Lin1992approximation, charikar2002constant}). Define $R_u$ to be the {\em fractional connection cost} for the point $u$, formally, $R_u = \sum_v d(u,v) ~z_{uv}$.
Now, construct a subset $S$ of the points $X$ as follows. Set $S = \emptyset$ and $T=X$ to begin with, and in every step, find $u \in T$ that has the smallest $R_u$ value (breaking ties arbitrarily) and add it to $S$. Then, remove all $v$ such that $d(u, v) \le 2R_v / \eps$ from the set $T$. Suppose we continue this process until $T$ is empty.

The set $S$ obtained satisfies the following property: for all $u, v \in S$, $d(u, v) \ge (2/\eps) \cdot \max \{R_u, R_v\}$. This is true because if $u$ was added to $S$ before $v$, then $R_u \le R_v$, and further, $v$ should not have been removed from $T$, which gives the desired bound. The property above implies that the set of metric balls $\{ B(u, R_u/\eps) \}_{u \in S}$ are all disjoint. 

Next, we observe that each such ball contains a fractional $y$-value (in the original LP solution) of at least $(1-\eps)$.  This is by a simple application of Markov's inequality. By definition,
$R_u = \sum_{v} d(u, v) ~z_{uv}$, and thus $\sum_{v \not\in B(u, R_u/\eps)} z_{uv} \le \eps$. This means that $\sum_{v \in B(u, R_u/\eps)} x_{uv} \ge 1-\eps$, and thus $\sum_{v \in B(u, R_u/\eps)} y_v \ge 1-\eps$. As the balls are disjoint, we have that $|S| \le k/(1-\eps)$. 

Now, consider an algorithm that opens all the points of $S$ as centers. By construction, all $v \not\in S$ are at a distance $\le 2R_v/\eps$ from some point in $S$, and thus for any group $i$, we have that 
$\sum_{u \in X_i} d(u, S) \le \sum_{u \in X_i} 2R_u/\eps \le 2\lambda/\eps$. Setting $\ybar$ to be the indicator vector for $S$ and $\zbar$ to be the assignment mapping each point to the closest neighbor in $S$, the theorem follows.
\end{proof} 

\begin{remark}[Extension to $k$-means]\label{remark:kmeans}
The argument above can be extended easily to obtain similar results for the $k$-means objective. We simply replace all distances with the squared distances. The metric ball around each point can be replaced with the $\ell_2^2$ ball, and the same approximation factors can be shown to hold.
\end{remark}

\paragraph{Beyond bi-criteria: an LP based heuristic.} While Theorem~\ref{thm:bicriteria} can achieve $\sum \ybar_v = k/(1-\eps)$ for any $\eps >0$, it is still weaker than what is possible for the standard $k$-median problem. The integrality gap instance shows that this is unavoidable in the worst case (using this LP). 
However, it turns out that there exist rounding algorithms for $k$-median that are faithful in expectation (as in Definition~\ref{def:faithful}), and end up with $\sum \ybar_v$ being precisely $k$. Specifically,
\begin{thm}[\cite{charikar2012dependent}]\label{thm:charikar-li}
There exists a rounding algorithm for \lpk{} that is $\alpha$-faithful in expectation with $\alpha \le 4$, and outputs precisely $k$ clusters.
\end{thm}


\begin{corr}
\label{corr:faithful}
Let $(z, y)$ be a solution to $\lpk{}$. There exists a rounding algorithm that (a) produces precisely $k$ clusters, and (b) ensures that the expected connection cost for every group is $\le 4 \lambda$.
\end{corr}
The corollary follows directly from Theorem~\ref{thm:charikar-li}, by linearity of expectation. This does not guarantee that the rounding {\em simultaneously} produces a small connection cost for all groups, but it gives a good heuristic rounding algorithm. In examples where every group has many points well-distributed across clusters, the costs tend to be concentrated around the expectation, leading to small connection costs for all clusters. We will see this via examples in the experiments section \ref{sec:experiment-kmedian}.

\subsection{Relaxation for \Ratio{}-Fair clustering}\label{sec:lp-ratio}
We now show that the rounding methods introduced in Section~\ref{sec:lp-error} can also be used for \Ratio{}-fair clustering. We consider the following auxiliary LP:

\begin{align}
\min~\lambda \quad &\text{ subject to} \notag\\  
\begin{aligned}
\sum_v z_{uv} &= 1 \quad \text{for all $u$}, \notag \\
z_{uv} &\le y_v \quad \text{for all $u,v$}, \notag \\
\end{aligned}&\qquad
  \begin{aligned}
\sum_{v \in X} y_v &\le k, \notag\\    
0 \le z_{uv},  ~y_v &\le 1 \quad \text{for all $u, v$},\notag\\
\end{aligned}\\
\sum_{u \in X_i} \sum_{v \in X} d(u, v) z_{uv} &\le \lambda \cdot \kmed (X_i) \text{ for all $i$.}
\label{eq:class-bound-ratio}
\end{align}
The constraint~\eqref{eq:class-bound-ratio} now involves a new term, $\kmed (X_i)$, which is an approximation to the optimum $k$-median objective of the set $X_i$.  For our purposes, we do not care how this approximation is achieved --- it can be via an LP relaxation~\cite{charikar2002constant, Li2013approximating}, local search~\cite{arya2004local, gupta2008}, or any other method.  We assume that if $\vartheta_i$ is the optimum $k$-median objective for $X_i$, then for all $i$, $\frac{\kmed(X_i)}{\vartheta_i} \le \rho$ for some constant $\rho$. (From the works above, we can even think of $\rho$ as being $\le 3$.)

\begin{lemma}\label{lem:ratio-approx}
Suppose there is a rounding procedure that takes a solution $(\lambda, z, y)$ to \lpratio{} and outputs a set of centers $S$ with the property that for some parameter $\eta$, 
\begin{equation}
\label{eq:ratio-class-bound} \sum_{u \in X_i} d(u, S) \le \eta \lambda \cdot \kmed(X_i) \quad \text{for all groups $i$}.
\end{equation}
Then, this algorithm provides an $\eta \cdot \rho$ approximation to \Ratio{}-fair clustering.
\end{lemma}
\begin{proof}
Let $\opt$ be the optimum value of the ratio-fair objective on the instance $\{X_1, \dots, X_m\}$.  The main observation is that the LP provides a lower bound on $\opt$. This is true because any solution to ratio-fair clustering leads to a feasible integral solution to \lpratio{}, where the RHS of the constraint~\eqref{eq:class-bound-ratio} is replaced by $\lambda \cdot \vartheta_i$.  Since $\vartheta_i \le \kmed(X_i)$, it is also feasible for \lpratio{}, showing that the optimum LP value is $\le \opt$.

Next, consider a rounding algorithm that takes the optimum LP solution $(\lambda^*, z^*, y^*)$ and produces a set $S$ that satisfies~\eqref{eq:ratio-class-bound} (with $\lambda^*$ replacing $\lambda$ on the RHS). Then, since $\kmed(X_i) \le \rho \vartheta_i$, we have
\[ \sum_{u \in X_i} d(u, S) \le (\rho \eta) \lambda^* \cdot \vartheta_i \quad \text{for all groups $i$},\]
and using $\lambda^* \le \opt$ completes the proof of the lemma.
\end{proof}

Thus, it suffices to develop a rounding procedure for \lpratio{} that satisfies~\eqref{eq:ratio-class-bound}. Here, we observe that the rounding from Theorem~\ref{thm:bicriteria} directly applies (because ensures that every $u \in X$, $d(u, S) \le 2 R_u/\epsilon$), giving us the same bi-criteria guarantee (and the same adjustment under faithful rounding). 

\begin{corr}[Corollary to Theorem~\ref{thm:bicriteria}]\label{corr:ratio}
For any $\eps>0$, there is an efficient algorithm that outputs $k/(1-\eps)$ clusters and achieves a $(6/\eps)$ approximation to the optimum  value of the \Ratio{} objective.
\end{corr}



\section{Facility Location}
\label{sec:facloc}
As we discussed earlier, the objective in $k$-means and $k$-median clustering measures how close the points are (on average) to their cluster centers. There can be situations in which this objective is more important than having a strict bound on the number of clusters produced. Consider the example we saw earlier, where points correspond to clients located in a metric space, and we open a {\em facility}  at the cluster center with the goal of {\em serving} the clients in the cluster. In such a context, it is reasonable to open more facilities if it serves clients better, as long as they collectively ``pay'' for opening. This is the motivation for the well-known facility location problem~\cite{cornuejols1983uncapicitated}. 

\begin{defn}\label{def:facility-location}
Let $X$ be a set of clients and $\calL$ be a set of locations in a metric space with distance function $d$. For each location $v$, we have an {\em opening cost} $f_v$, which is the cost of opening a facility at $v$. The objective is now a sum of the {\em connection costs} and the {\em opening costs}. Formally, the goal is to select a subset $L$ of $\calL$, so as to minimize
\[ \sum_{u \in X} d(u, L) + \sum_{v \in L} f_v. \]
\end{defn}

Now consider the setting in which clients fall into different demographic groups. We propose our objective based on an {\em equal division} of the facility opening costs among all the clients. Thus, every client pays an opening cost of $\frac{1}{|X|} \sum_{v \in L} f_v$. Our definition of fair facility location aims to ensure that the average total cost (opening plus connection) is small for all the groups.

\begin{defn}\label{def:fair-facility}
Let $X$ be a set of clients composed of groups $X_1, X_2, \dots, X_m$, and let $\calL$ be a set of locations. The goal of {\em fair facility location} is to select a subset $L$ of the locations, so as to minimize
\[ \max_{i} \left\{ \frac{1}{|X_i|} \sum_{u \in X_i} d(u, L) +
\frac{1}{|X|} \sum_{v \in L} f_v \right\}.\]
\end{defn}

We remark that the second term in the objective is independent of the group (as this is the cost paid by every client). 



Our first result is a constant factor approximation algorithm.

\begin{theorem}\label{thm:facility}
There is an efficient (polynomial time) algorithm for fair facility location with an approximation ratio of 4.
\end{theorem}
\begin{proof}
The proof turns out to follow easily from classic results on facility location. Consider the following linear program:

\begin{align}
 \min  ~\lambda + \frac{1}{|X|} \sum_v y_v \cdot f_v \quad & \text{ subject to} \notag\\  
  \begin{aligned}
\sum_v z_{uv} &= 1 \quad \text{for all $u$}, \notag \\
z_{uv} &\le y_v \quad \text{for all $u,v$}, \notag \\
\end{aligned}
  &\qquad
  \begin{aligned}
   \notag\\
0 \le z_{uv},  ~y_v &\le 1 \quad \text{for all $u, v$}.\notag \\
\end{aligned}\\
  \frac{1}{|X_i|} \cdot \sum_{u \in X_i} \sum_{v \in X} d(u, v) z_{uv} &\le \lambda \quad \text{for all groups $i$},\label{eq:per-group-facility} 
\end{align}

Now we note the rounding algorithm from~\cite{shmoys1997approximation} is a $4$-faithful rounding procedure (as in Definition~\ref{def:faithful}), and also ensures that for the produced integral solution $\ybar$, $\sum_v \ybar_v f_v \le 4 \sum_v y_v f_v$. Using this, the desired approximation factor follows.
\end{proof}

It is an interesting open problem to improve the approximation ratio above. We note that there are many better approximation algorithms for facility location (see, e.g.,~\cite{Li2011A} and references therein). However, many of these algorithm are either faithful only in expectation, or they use primal-dual rounding schemes which do not seem applicable in our context.

\subsection{Ensuring uniform load for facilities}
\label{ssec:load}
Definition~\ref{def:fair-facility} captures the requirement that all the demographic groups have a small cost for accessing the opened facilities. However, this objective does not fully capture real life constraints. Consider the example in Figure~\ref{fig:dense-and-sparse}.
\begin{figure}[!h]
\begin{center}
\includegraphics[scale=0.25]{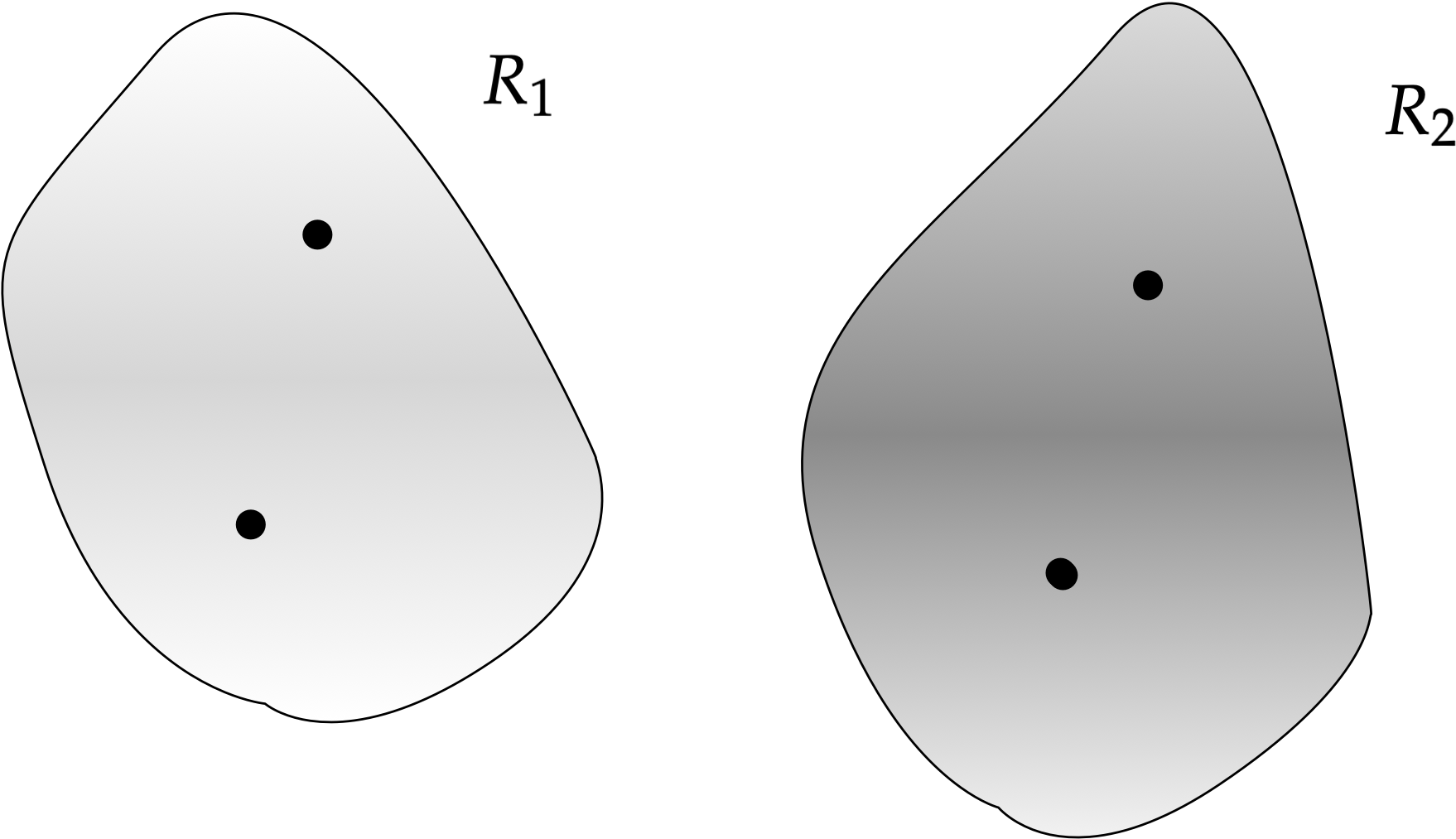}
\end{center}
\caption{Geometrically similar regions $R_1$ and $R_2$, with vastly different population density.}\label{fig:dense-and-sparse}
\end{figure}
Suppose that $R_1$ has a population density of $\delta$ and $R_2$ has a population density of $100 \delta$. Suppose $R_1$ consists all the people of the first group, $X_1$, and $R_2$ has people of $X_2$. Now, opening $2$ facilities in each region is a better solution (in the objective above), compared to opening $3$ in $R_2$ and $1$ in $R_1$, because the connection cost term in the objective is normalized by the size of the group. (This is justified in order to avoid an unfair treatment for a smaller group.)

However, in a time-sensitive application (such as polling), having a lot of clients allocated to a facility can lead to over-crowding, and thus a loss of utility. One way to repair this is by adding a {\em delay} term to the utility. This leads to a quadratic objective function, which appears difficult to optimize. Instead, we propose enforcing load-balance using a {\em capacity} constraint for facilities, the simplest of which is that for some $U$, each facility can serve at most $U$ clients (in total, across all the groups). In practice, these are not hard constraints, and it is reasonable for an algorithm to violate it by a small factor.


Most of the standard clustering formulations (including facility location) have been studied in the presence of capacity constraints. Constant factor approximations are known, both using local search~\cite{arya2004local} and LP relaxations~\cite{an2017facility}. 
In our setting with multiple groups, we show that the LP methods can be adapted to obtain approximation algorithms. 

\paragraph{Notation.}  As before, the set of clients $X$ is composed of $m$ groups $X_1, X_2,  \dots, X_m$ and the set of valid locations for facilities is $\calL$. They lie in a metric space whose distance function is denoted by $d$. $U$ will always denote the capacity bound.
A solution will consist of a subset $S \subseteq \calL$ and an assignment function $g : X \mapsto S$. The quantity $\opt$ will denote
\[ \min_{S,~ g} ~\frac{1}{|X|} \sum_{i \in S} f_i + \max_{t \in [m]} \frac{1}{|X_t|} \sum_{x \in X_t} d(x, g(x)), \]
where the minimization is over $g$ such that $|g^{-1}(s)| \le U$ for all $s \in S$. We prove that it is possible to achieve a constant factor approximation, as long as there is a constant (any fixed constant $>1$) slack in the capacity constraint.

\begin{theorem}\label{thm:facility-capacity}
Let $X, \calL, d, U$ be defined as above, and let $\opt$ denote the optimum objective value. Then for every $\eps > 0$, there is an efficient (polynomial time) algorithm that finds a solution $S, g$ with the following properties:
\begin{enumerate}
    \item the objective value is $\le O_{\epsilon} (1) \cdot \opt$. 
    \item for every $s \in S$, we have $|g^{-1}(s)| \le (1+\eps)~ U$.
\end{enumerate}
\end{theorem}

\paragraph{Remark.} It is an interesting open problem to study the case of \emph{hard} capacity constraints. Likewise, we are assuming that the capacity bound $U$ is independent of the facility. Modifying the rounding algorithm to allow different capacities at different locations is also an interesting direction. We note that both can be achieved using LP methods for standard facility location (e.g.,~\cite{an2017facility}), but known rounding algorithms are not faithful to the best of our knowledge, thereby making it tricky to apply to the setting with multiple groups. Meanwhile, we note that in our motivating applications, both uniform capacities and {\em soft} constraints are reasonable assumptions. 

\begin{proof}
The proof follows along the lines of~\cite{shmoys1997approximation}, but we slightly adapt it to get a stronger bound on $|g^{-1}(s)|$. We start by solving a slight variant of the LP~\eqref{eq:per-group-facility}, where we add the capacity constraint:
\begin{equation}
\sum_{u \in X} z_{uv} \le U \cdot y_v \qquad \text{for all $v \in \calL$}.
\label{eq:facility-capacity}
\end{equation}
In an integer solution, the RHS is zero if $y_v=0$ ($v$ is not opened) and $U$ if $y_v=1$ ($v$ is opened), as intended. Starting with the optimal fractional solution for this LP, the rounding algorithm is then as follows:
\begin{algorithm}
\caption{FFL-Rounding ($z, y$), parameters $\theta, \delta \in (0,1)$}
\label{alg:rounding-facility}
\begin{algorithmic}[1]
\State Perform Filtering with parameter $\theta$ on $(z, y)$ to obtain $(z', y')$ \State Define $F = \{v \in \calL : y_v' \ge 1/2\}$, update $y'_v = 1$ for all $v \in F$
\State\label{step:c-delta} Define $C_\delta = \{u \in X : \sum_{v \in F} z'_{uv} < (1-\delta) \}$
\While {$C_\delta \neq \emptyset$}
\State Let $u$ be the client in $C_\delta$ with smallest $R_u$
\State Let $S$ be the set of facilities in $B(u, \frac{R_u}{\theta})$ with $y'_v \in (0,1)$
\State Let $r = \lceil \sum_{v \in S} y'_v \rceil$
\State Open the $r$ cheapest facilities in $S$ (call this set $O$)
\State Update the set $F$, setting $F \leftarrow F \cup O$
\State Update $y_v$ values, setting $y_v = 1$ for $v \in O$ and $y_v =0$ for $v \in S \setminus O$
\State \label{step:move-assign} Move all fractional assignment $z_{uv}'$ from $S \setminus O$ to $O$
\State Update the set $C_\delta$ using the definition in step~\ref{step:c-delta}
\EndWhile
\State Close all remaining fractionally open facilities
\State Let $\beta_u = \sum_{v \in F} z'_{uv}$. If $\beta_u < 1$, re-scale all $z'_{uv}$ by $1/\beta_u$
\State Use bipartite matching to round fractional assignment to an integral one, denoted $g$
\State return $F, g$
\end{algorithmic}
\end{algorithm}

The first step is filtering~\cite{shmoys1997approximation}, where we convert a feasible fractional solution $(z, y)$ to the LP to another fractional solution $(z', y')$, with the additional constraint that if $z'_{uv} > 0$ (i.e., client $u$ has a non-zero fractional assignment to facility $v$), then $v \in B(u, \frac{R_u}{\theta})$. This fractional solution satisfies all the constraints of the LP except the capacity constraint~\eqref{eq:facility-capacity}, and also satisfies~\eqref{eq:facility-capacity} up to a slack factor of $\frac{1}{(1-\theta)}$ on the RHS. The parameters $\theta, \delta > 0$ used here will be chosen later. In the process, the opening cost term in the objective increases by a factor at most $1/(1-\theta)$.

The next step is to consider all the facilities $v \in \calL$ with $y_v' \ge 1/2$, and set $y_v' = 1$, and add all such facilities to the {\em opened} set $F$. Again, this step only increases the opening cost term in the objective by a factor of 2.

At every point of time, the algorithm maintains a set $C_\delta$, which is the set of clients $u$ such that
\[ \sum_{v \in F} z'_{uv} < (1-\delta).\]
In other words, $u$ has a significant ($\ge \delta$) fractional assignment to unopened facilities. As long as $C_\delta$ is non-empty, the algorithm chooses $u \in C_\delta$ with the smallest value of $R_u$. It then opens facilities in the vicinity of $u$. Define $B_u = B(u, \frac{R_u}{\theta})$, for convenience. Suppose $S$ is the set of fractionally open facilities in $B_u$. Because $u \in C_\delta$, we have that $\sum_{v \in S} z_{uv} > \delta$, which in turn implies that $\sum_{v \in S} y_v' > \delta$ (because of the LP constraint $y_v' \ge z_{uv}$ for all $u$). 

The algorithm opens $r = \lceil \sum_{v \in S} y_v' \rceil$ centers from $S$, of the least cost. We claim that in the process, the sum $\sum_{v \in S} f_v y_v'$ increases by at most $\max\{2, 1/\delta\}$. This is easy to see by considering two cases: if $\sum_{v \in S} y_v' \le 1$, then the algorithm opens exactly one center, and the cost increases by at most $1/\delta$. If the sum is $>1$, then since all the $y_v'$ values were $<1/2$, there must have been $> r$ non-zero terms in the summation, and we can argue that the cost increase is at most $2$ (see~\cite{shmoys1997approximation}). As our choice of $\delta$ will be $<1/2$, the $(1/\delta)$ term dominates. We thus have that this step of the algorithm increases only the facility opening cost, and by a factor at most $1/\delta$. 

The next step is the re-allocation of clients from $S \setminus O$ to $O$. Let $u'$ be one such client. If $u' = u$, we simply note that for all $v \in O$, $d(u, v) \le \frac{R_u}{\theta}$, and thus all the (fractional) demand is still routed to a facility at distance at most $\frac{R_u}{\theta}$ away. Likewise, if $u' \ne u$, then because $R_u \le R_{u'}$, demand is still routed to a facility at distance at most $\frac{3R_{u'}}{\theta}$ away.

Following this, we are only left with clients most of whose demand (at least $(1-\delta)$ fraction) has already been routed to facilities in $F$. The scaling by $\beta_u$ (as defined in the algorithm) increases the objective by a factor at most $1/(1-\delta)$.

The steps above together help obtain a solution in which $y'$ is integral, but the $z_{uv}'$ may be fractional. Also, the connection cost term in the objective is scaled by at most
\[ \frac{3}{\theta} \cdot \frac{1}{(1-\delta)}. \]
Moreover, this bound holds for every point (thus the rounding is faithful). 
The facility opening cost is scaled by at most
\[ \frac{1}{1-\theta} \cdot 2 \cdot \frac{1}{\delta}.\]
The capacity constraint is violated by a factor of $\frac{1}{(1-\theta)(1-\delta)}$. 

The final step of the rounding is bipartite matching. Here, since the demands are all unit, we can show that it is possible to convert the fractional assignment into an integral one (by solving an instance of the transportation problem, see~\cite{shmoys1997approximation}).  Setting $\theta = \delta = \epsilon/3$, the result follows.
\end{proof}
\section{Experiments}
\label{sec:experiments}
 In the first two parts of this section we evaluate the proposed fair $k$-median and fair facility location algorithms and provide an empirical assessment for their performance. In the final part, we compare the balance-based approach to fair clustering with respect to our representation-based notions. Throughout the experiments, we consider five datasets:
\begin{itemize}
\item \textbf{Synthetic}. Synthetic dataset with three features. First feature is binary (``majority" or ``minority"), and determines the group example belongs to. Second and third attributes are generated using distribution $\mathcal{N}(0,0.5^2)$ in the majority group, and distribution $\mathcal{N}(3,0.5^2)$ in minority group. Majority and minority groups are of size 250 and 50, respectively.
\item \textbf{Iris}.\footnote{\href{https://archive.ics.uci.edu/ml/datasets/iris}{https://archive.ics.uci.edu/ml/datasets/iris}} Data set consists of 50 samples from each of three species of Iris: Iris setosa, Iris virginica and Iris versicolor. Selected features are length and width of the petals.
\item \textbf{Census}.\footnote{\href{https://archive.ics.uci.edu/ml/datasets/adult}{https://archive.ics.uci.edu/ml/datasets/adult}} Dataset is 1994 US Census and selected attributes are ``age",``fnlwgt", ``education-num", ``capital-gain" and ``hours-per-week". groups of interests are ``female" and ``male".
\item \textbf{Bank}.\footnote{\href{https://archive.ics.uci.edu/ml/datasets/Bank+Marketing}{https://archive.ics.uci.edu/ml/datasets/Bank+Marketing}} The dataset contains records of a marketing campaign based on phone calls, ran by a Portuguese banking institution. Selected attributes are ``age", ``balance", ``duration" and groups of interest are ``married" and ``single".
\item \textbf{North Carolina voters}. \footnote{\href{https://www.ncsbe.gov/results-data/voter-registration-data}{https://www.ncsbe.gov/results-data/voter-registration-data}} In this dataset, we are interested in ``latitude'' and ``longitude'' values of each voter's residence, and use ``race'' attribute to identify different demographic groups.
\end{itemize}

We do not evaluate the capacitated version of fair facility location that we discuss in Section~\ref{ssec:load}. This algorithm is more complex and is beyond the scope of this work. We should also note that throughout this section we compare the results of proposed fair algorithms to standard $k$-median, which is implemented using the standard linear program formulation.

\subsection{Fair $k$-median}
\label{sec:experiment-kmedian}
In this section we employ two algorithms to compute $k$-median clusterings which are group-representative. We call these algorithms LP-Fair $k$-median and LS-Fair $k$-median.

\paragraph{LP-Fair $k$-median}
LP-Fair $k$-median first solves the FairLP linear program presented in Section\ref{sec:lp-ratio}. Since it is not possible to compare the results of bi-criteria algorithm to standard $k$-median algorithms due to the varying number of centers, we chose to obtain integral solutions via the faithful rounding procedure described in Corollary \ref{corr:faithful}. The rounding is based on the matching idea proposed by Charikar et al. \citep{charikar2012dependent}, and is done in four phases:
\begin{description}
\item[Filtering:] Similar to the filtering technique described in section \ref{sec:lp}, we construct a subset $S$ of the points with a small adjustment that after adding a point $u$ to the set $S$, all points $v$ from the original set such that $d(u, v) \le 4R_v$ will not be considered to be added to $S$ anymore.
\item[Bundling:] For each point $u \in S$, we create a bundle $B_u$ which is comprised of the centers that exclusively serve $u$. In the rounding procedure, each bundle $B_u$ is treated as a single entity, where at most one center from it will be opened. The probability of opening a center from a bundle, $B_u$, is the sum of $y_c, c\in B_u$, which we call bundle's volume.
\item[Matching:] The generated bundles have the nice property that their volume lies within $1/2$ and $1$. So given any two bundles, at least one center from them should be opened. Therefore, while there are at least two unmatched points in $S$, we match the corresponding bundles of the two closest unmatched points in $S$.
\item[Sampling:] Given the matching generated in the last phase, we iterate over its members and consider the bundle volumes as probabilities, to open $k$ centers in expectation.
\end{description}
The centers picked in the sampling phase are returned as the final $k$ centers.

\paragraph{LS-Fair $k$-median}
In this section, we propose a heuristic local search algorithm in addition to LP-Fair $k$-median. Arya et al. proposed a local search algorithm to approximately solve the $k$-median problem \citep{arya2004local}. Their algorithm starts with an arbitrary solution, and repeatedly improves it by swapping a subset of the centers in the current solution, with another set of centers not in it. We modify this algorithm to minimize the maximum average cost over all groups. Assuming we're given a cost function and $X_1, X_2, \dots, X_m$ as groups where $X = \cup_i X_i$, LS-Fair $k$-median is presented in Algorithm \ref{alg:lsfair}.

\begin{algorithm}
\caption{LS-Fair $k$-median($k,cost,X,X_1, \dots, X_m$)}
\label{alg:lsfair}
\begin{algorithmic}
\State $S \leftarrow$ an arbitrary set of $k$ centers from $X$
\State $old\_cost \leftarrow inf$
\State $new\_cost \leftarrow max(cost_S(X_1),\dots,cost_S(X_m))$
\While{there is $t' \in X$ and $t \in S$ s.t. $max(cost_{S\char`\\t \cup t'}(X_1),\dots,cost_{S\char`\\t \cup t'}(X_m)) < old\_cost$)}
\State $S \leftarrow S\char`\\t \cup t'$
\State $old\_cost \leftarrow max(cost_{S\char`\\t \cup t'}(X_1),\dots,cost_{S\char`\\t \cup t'}(X_m))$
\EndWhile
\State return $S$
\end{algorithmic}
\end{algorithm}
Unlike the LP-Fair $K$-Medians, we do not provide any theoretical bounds on LS-Fair $K$-Median. In fact, the following example shows that LS-Fair $K$-Medians algorithm with the \error{}-fair objective can have local optima that are arbitrarily worse than the global optimum.

\paragraph{Description of the instance.} Let $A$ and $B$ two sets that are far apart (think of the distance between any pair as $M \rightarrow \infty$). $A = A_1 \cup A_2$, where $|A_1| = 1$ and $|A_2| = t$, for some integer parameter $t$. Likewise, suppose that $B = B_1 \cup B_2$, of sizes $1, t$ respectively. Suppose that all the elements of $A_2$ (so also $B_2$) are at distance $\eps$ away from one another. Suppose the distance between $A_1$ and $A_2$ (so also $B_1$ and $B_2$) is $d$. 

Now, suppose the two groups are $X_1 = A_1 \cup B_2$ and $X_2 = B_1 \cup A_2$.  Let $k=2$. The optimal solution is to choose one point in $A_2$ and another in $B_2$.  This results in an objective value of
\[ \max \{ t \eps + d, t\eps + d \} = t\eps +d. \]

Consider the solution $\{a_1, b_1\}$ that chooses the unique points from $A_1$ and $B_1$. The $k$-median objective for both the groups is $td$, and thus the \error{}-fair objective is $td$.  Now, consider swapping $a_1$ with some point $x \in A_2$. This changes the $k$-median objective for group 1 from $td$ to $td+\eps$, and so even though the swap significantly decreases the objective for the second group, the local search algorithm will not perform the swap. The same argument holds for swapping $b_1$ with a point $y \in B_2$. It is thus easy to see that $\{a_1, b_1\}$ is a locally optimum solution.

However, the ratio between the \error{}-fair objectives of this solution and the optimum is $\frac{td}{t \eps + d} \approx t$ for $\eps \rightarrow 0$. Thus the gap can be as bad as the number of points.

\paragraph{Results.} In this experiment, in order to save space, we focus on the Census and Bank datasets. However, we consider two subsamples of each dataset: 1:1 Census contains 150 female and 150 male examples, 1:5 Census contains 50 female and 250 male examples, 1:1 Bank contains 150 married and 150 single examples, and 1:5 Bank contains 50 married and 250 single examples. The results are summarized in table \ref{ta:heuristic}. Group-optimal presents the optimal average cost for a group, when it is clustered by itself via $k$ centers. $k$-median presents a group's average cost, in a clustering generated by the standard $k$-median algorithm performed on all groups together. The other rows in the table show the percentage increase/decrease in costs, for either of the described fair algorithms, using the two cost functions. In general, the results demonstrate the effectiveness of our algorithms. However, we emphasize on the difference between 1:1 and 5:1 samples. In the 1:1 case, the groups have the same size and standard $k$-median treats them roughly the same. But in the 5:1 case, if the groups have different distributions, standard $k$-median favors majority group over the other, and the effectiveness of our proposed algorithms are more evident. We should note that in all experiments, points were clustered using 3 centers.\footnote{Each dataset was sampled 10 times and we reported the overall average.}

\begin{table*}[htbp]
\caption{Clustering Bank and Census datasets using LS-Fair and LP-Fair algorithms. Standard $k$-median and group optimal
rows present the actual groups’ average costs. \error and \Ratio rows are the percentage increase/decrease in group costs
for proposed algorithms, compared to the corresponding values in standard $k$-median and group optimal rows, respectively.}
\label{ta:heuristic}
\centering
\begin{tabular}{@{}llcccccccc@{}}
\toprule
\multicolumn{2}{l}{\multirow{2}{*}[-0.25em]{Datasets}} & \multicolumn{2}{c}{\textit{1:1 Census}} & \multicolumn{2}{c}{\textit{1:5 Census}} & \multicolumn{2}{c}{\textit{1:1 Bank}} & \multicolumn{2}{c}{1:5 Bank} \\ \cmidrule(lr){3-4} \cmidrule(lr){5-6} \cmidrule(lr){7-8} \cmidrule(lr){9-10} 
\multicolumn{2}{c}{} & female & male & female & male & married & single & married & single \\ \midrule
\multicolumn{2}{l}{Standard $k$-median} & 35264 & 32351 & 40212 & 32689 & 596 & 730 & 948 & 665 \\ \addlinespace[5pt]
\multirow{2}{*}{\error (\%)} & LS-Fair & 98.7 & 102.9 & 94.2 & 110.5 & 105.2 & 98.3 & 79 & 111.2 \\
 & LP-Fair & 98.3 & 105 & 95.5 & 109.1 & 105.7 & 98.2 & 78 & 114.7 \\ \addlinespace[5pt] \midrule
 \multicolumn{2}{l}{Group optimal} & 34499 & 31528 & 35349 & 32619 & 569 & 686 & 659 & 655 \\ \addlinespace[5pt]
\multirow{2}{*}{\Ratio (\%)} & LS-Fair & 102.5 & 102.5 & 107.7 & 106.3 & 107.3 & 105.9 & 113 & 114 \\
 & LP-Fair & 102.6 & 102.5 & 107.6 & 103.8 & 107.7 & 105.2 & 116.3 & 113.4 \\ \bottomrule
\end{tabular}
\end{table*}

\begin{table*}[htbp]
\caption{Effects of enforcing balance on group representations}
\label{tab:balance}
\centering
\begin{tabular}{@{}lcccccccc@{}}
\toprule
\multicolumn{1}{l}{\multirow{2}{*}[-0.25em]{Datasets}} & \multicolumn{2}{c}{Synthetic} & \multicolumn{2}{c}{Iris} & \multicolumn{2}{c}{Census} & \multicolumn{2}{c}{Bank} \\ \cmidrule(lr){2-3} \cmidrule(lr){4-5} \cmidrule(lr){6-7} \cmidrule(lr){8-9}
\multicolumn{1}{c}{}	&	majority	&	minority	&	Setosa    & Versicolor   & female       & male        & married     & single     \\ \midrule
Standard $k$-median	& 0.514		& 0.678		& 0.169     & 0.256        & 34492     & 35083    & 627      & 682     \\
Balanced $k$-median	& 0.430 	& 3.476		& 0.101     & 2.819        & 34019     & 35876    & 622      & 694     \\ \bottomrule
\end{tabular}
\end{table*}

\subsection{Fair facility location}
In this section we empirically evaluate the algorithm presented in section \ref{sec:facloc}. We also use the 4-Faithful rounding procedure proposed in \citep{shmoys1997approximation} to obtain an integral solution. We use the data for voters in the state of North Carolina,  specifically Brunswick county. This dataset contains race and ethnicity of each voter as well as latitude and longitude values of their residence. In this experiment we focus on black and white demographic groups, which roughly constitute 7000 and 55000 voters, respectively.  As for the facilities, we assume each voter's residence could be used as a drop-off location. Therefore, in all experiments we use regular $k$-means clustering to select 100 locations out of the total 62000 data points as the set of facilities. We also assume the setup cost for all facilities are equal. The results of this experiment are presented in Figure \ref{fig:facility_results}, for different values of the facility setup cost. The results show that fair facility location algorithm lowers the average distance to polling locations for the worse off group, namely black voters, in comparison to standard version. Also, as we increase the setup cost for facilities, since fewer number of facilities will be opened, the average distance grows larger for both groups. We should also note that by opening fewer number of facilities, it will be harder for the algorithm to find a fair solution as it is more restricted. This is apparent from the larger discrepancies between the two groups in the fair(er) solution for higher values of setup cost.

\begin{figure}
\centering
\includegraphics[width=\columnwidth]{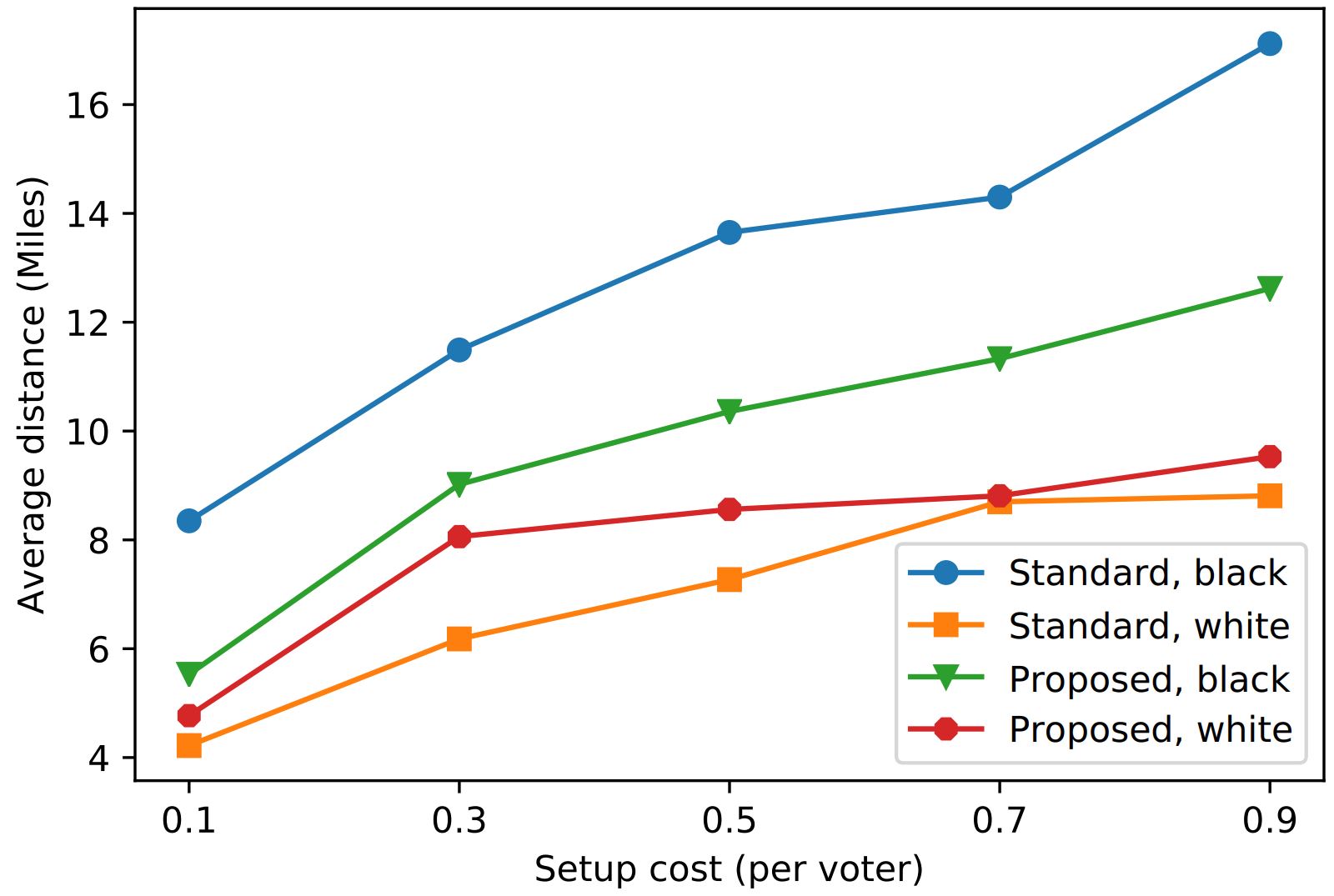}
\caption{Average distance to polling location for black and white voters}    
\label{fig:facility_results}
\end{figure}

\subsection{On balance and representations}
In this section, we empirically study the effects of enforcing balance on group representations. More specifically, we compare each group's average cost for standard $k$-median to the corresponding value under balance constraint. As for the balance-fair $k$-median, we chose to use the algorithm proposed by Backurs et al. \citep{backurs2019scalable}.\footnote{Implementation could be found \href{https://github.com/talwagner/fair_clustering}{here}.} In this experiment, we used the entire Synthetic and Iris datasets, and sampled 300 examples from each of Census (150 male, 150 female) and Bank (150 married, 150 single) datasets. In table~\ref{tab:balance}, we present the average costs for all groups within each dataset, in two clusterings generated by standard $k$-median and balanced $k$-median. In all datasets, we observe enforcing balance amplifies representation disparity across groups and leads to a higher maximum average cost. However, it is especially more noticeable in Synthetic and Iris datasets, where different groups have vastly different distributions.\footnote{The algorithm proposed by Backurs et al. works on only two groups. We chose two groups out of three from Iris. Repeating the experiment with other groups lead to similar results.}


\section{Conclusion}
In this work we presented a novel approach to think of and formulate fairness in
clustering tasks, based on group representativeness. Our main contributions are
introducing a fairness notion which parallels the development of fairness in
classification setting, proposing bicritera approximation algorithms for
$k$-medians under different variations of this notion, as well as approximation algorithms for facility location problem and providing theoretical
bounds for both. Our results suggest that our formulation provides better quality
representations especially when the groups are skewed in size. 


\section{Acknowledgement}
We would like to thank our colleagues Sorelle Friedler and Calvin Barrett from Haverford College, who provided us with voter location data in the state of North Carolina.

\bibliography{refs}
\bibliographystyle{abbrv}

\end{document}